\documentclass{article}

\usepackage[preprint] {neurips_2023}
\author{Sagar Shrestha, ~Rajesh Shrestha, ~Tri Nguyen, ~ and ~ Subash Timilsina}

\usepackage[utf8]{inputenc} 
\usepackage[T1]{fontenc}    
\usepackage{hyperref}       
\usepackage{url}            
\usepackage{booktabs}       
\usepackage{amsfonts}       
\usepackage{nicefrac}       
\usepackage{microtype}      
\usepackage{xcolor}         
\usepackage{graphicx} 
\usepackage{multirow}
\usepackage{amsmath}
\usepackage{amsthm}
\usepackage{amsfonts}
\usepackage{caption}
\usepackage{subcaption}

\usepackage{xcolor}
\definecolor{orange}{RGB}{255,107,0}
\definecolor{green}{RGB}{50,170,50}
\definecolor{purple}{RGB}{255,0,255}

\usepackage{tikz}
\usetikzlibrary{bayesnet}
\usetikzlibrary{arrows}
\usepackage{color}
\usepackage{graphicx}
\usepackage{caption}
\usepackage{wrapfig}
\usetikzlibrary{backgrounds}

\newtheorem{Fact}{Fact}

\title{Distribution Matching via Generalized Consistency Models}
\date{February 2024}

\usepackage{steinmetz}

\newcommand{\f}{\boldsymbol{f}}

\renewcommand{\v}{\boldsymbol{v}}
\newcommand{\g}{\boldsymbol{g}}
\newcommand{\h}{\boldsymbol{h}}

\newcommand{\s}{\boldsymbol{s}}

\newcommand{\x}{\boldsymbol{x}}
\newcommand{\z}{\boldsymbol{z}}

\newcommand{\J}{\boldsymbol{J}}
\newcommand{\A}{\boldsymbol{A}}

\newcommand{\cL}{\mathcal{L}}

\newcommand{\cN}{\mathcal{N}}

\newcommand{\bbR}{\mathbb{R}}
\newcommand{\bbE}{\mathbb{E}}

\usepackage{extarrows}

\newtheorem{proposition}{Proposition}

\DeclareMathOperator*{\minimize}{\textrm{minimize}}

\usepackage{xcolor}
\usepackage{psfrag,framed}
\usepackage{lipsum}
\usepackage{chapterbib}
\PassOptionsToPackage{normalem}{ulem}
\usepackage{ulem}
\definecolor{shadecolor}{RGB}{220,220,220}

\newcommand{\norm}[1]{\left \lVert #1  \right \rVert}

\newcommand{\cG}{\mathcal{G}}
\newcommand{\bm}{\boldsymbol}

\begin{document}

\maketitle

\begin{abstract}
Recent advancement in generative models have demonstrated remarkable performance across various data modalities. Beyond their typical use in data synthesis, these models play a crucial role in distribution matching tasks such as latent variable modeling, domain translation, and domain adaptation. Generative Adversarial Networks (GANs) have emerged as the preferred method of distribution matching due to their efficacy in handling high-dimensional data and their flexibility in accommodating various constraints. However, GANs often encounter challenge in training due to their bi-level min-max optimization objective and susceptibility to mode collapse. In this work, we propose a novel approach for distribution matching inspired by the consistency models employed in Continuous Normalizing Flow (CNF). Our model inherits the advantages of CNF models, such as having a straight forward norm minimization objective, while remaining adaptable to different constraints similar to GANs. We provide theoretical validation of our proposed objective and demonstrate its performance through experiments on synthetic and real-world datasets.
\end{abstract}

\section{Introduction}
Continuous Normalizing Flow (CNF) \cite{chen2018neural, lipman2022flow, liu2022flow, albergo2023stochastic} is a class of generative models based on continuous time evolution of probability density from the source distribution to target distribution. Diffusion model is an instance of CNF that connects a simple source distribution to a given target distribution via the density path prescribed by the diffusion process \cite{song2020score}. Diffusion model has established itself as the state-of-the-art method for data synthesis in many modalities such as images, videos, audio, protein structures, etc. Nevertheless, the slow iterative sampling procedure and rigid diffusion path have motivated a plethora research focused on efficient sampling and flexible density path design \cite{song2023consistency, dou2024unified, lipman2022flow}. Flow matching \cite{lipman2022flow, liu2022flow, albergo2023stochastic} has emerged as a promising direction that generalizes diffusion models to allow for flexible density path, such as optimal transport path, and arbitrary source distribution. 

Apart from data synthesis, a common application of generative model is to match distribution between two or more random variables. Such distribution matching problems arise ubiquitously in machine learning, e.g., domain translation, domain adaptation, latent variable models, inverse problems, etc. Such distribution matching tasks are often tackled using moment matching \cite{li2015generative}, generative adversarial networks (GAN) \cite{goodfellow2014generative} etc. GANs stand as the preferred method for distribution matching because of its success in high dimensional data and high sample quality. Nonetheless, GANs present a difficult min-max optimization problem which has been known to be notoriously difficult to optimize \cite{salimans2016improved}. Moreover, GANs suffer from mode collapse issue, that results in poor sample diversity, which can be detrimental to some applications. 

Although recent CNFs (diffusion and flow matching) provide an attractive quadratic minimization objective with stable and scalable training, it is not clear how one can use CNFs as a distribution matching tool like GANs. The reason is two-fold: (i) CNFs learn a unique mapping from source to target of the same dimensions specified by the choice of vector field (e.g., diffusion), (ii) CNFs do not provide any (pseudo) measure of distribution divergence which can be used as a signal for optimization. In contrast, GANs can learn any feasible mapping between the source and target domains.

In this work, we propose a method for distribution matching with CNFs. By leveraging the recently proposed consistency models that learn a one-step generative model based on CNF paths, our method results in a quadratic objective for general distribution matching tasks. Theoretical analysis is presented to show the correctness of the proposed objective. Some small scale synthetic and real data experiments are presented to validate our claims.  



\section{Background}

\subsection{Continuous Normalizing Flow (CNF)}

Given a source distribution $\rho_0$ defined on $\bbR^N$ and a target distribution $\rho_1$ also on $\bbR^N$, a CNF attempts to find a time-dependent and time-differentiable mapping $\bm \psi_t: \bbR \times \bbR^N \to \bbR^N, ~t \in [0,1]$, termed as ``flow'', such that $[\bm \psi_1]_{\# \rho_0} = \rho_1$ and $[\bm \psi_0]_{\# \rho_0} = \rho_0$. Here, the notation $[\bm \psi_t]_{\# \rho_0}$ denotes the push-forward distribution of $\rho_0$ via $\bm \psi_t$, i.e., a random variable $\x \sim \rho_0$ follows $\bm \psi_t(\x) \sim [\bm \psi]_{\# \rho_0}$. The time-derivative of the flow $\bm \psi_t$ is a time-dependent vector field $\v_t: \bbR^N \to \bbR^N$, defined as follows:
\begin{align}
    \v_t(\bm \psi_t(\x)) = \frac{d}{dt} \bm \psi_t(\x).
\end{align}
The time-dependent flow $\bm \psi_t$ results in a time-dependent transformed distribution. This will induce a time-dependent density path $\bm \rho_t := [\bm \psi_t]_{\# \rho_0}$. 

Recently popular CNFs are diffusion models and flow matching. Generally, these CNFs estimate $\v_t$ in some form, such that samples from $\rho_1$ (or $\rho_0$) can be generated by following the vector field $\v_t$ (or $-\v(t)$) as follows:
\begin{align}\label{eq:sampling_with_v}
    \x_1 = \x_0 + \int_0^1 \v_t(\x_t) dt.
\end{align}


\subsection{Diffusion and Flow Matching}

\textbf{Diffusion Models.}
A diffusion model starts from clean data density $\rho_1$, progressively perturbs the data with Gaussian noise to obtain noisy data density $\rho_t$ at time $t$, such that the final density $\rho_0$ is a pure Gaussian noise. There are many variants of diffusion paths $\rho_t$ depending upon the noise adding schedules, such as variance preserving (VP), variance exploding (VE), and sub-VP \cite{song2020score}. Here, we present the widely used VE variant \cite{karras2022elucidating}. The density-path traced by this diffusion process is:
\begin{align}
    \rho_t(\x) = (\rho_1 \ast \cN(0,\sigma^2(t) \bm I))(\x)
\end{align}
where $\ast$ denotes the convolution operator and $\sigma(t): [0,1] \to \bbR$ is a monotonically decreasing function such that $\sigma(1) = 0$, and $\sigma(0)$ is large enough that $\rho_1 \ast \cN(0,\sigma^2(0) \bm I) \approx \cN(0, \sigma^2(0) \bm I)$. A vector field that follows the density path prescribed by diffusion model is called the probability flow ODE and is expressed as follows:
\begin{align}\label{eq:diffusion_vf}
    \v_t(\x) = \sigma(t) \nabla_{\x} \log \rho_t(\x),
\end{align}
where $\nabla_{\x} \log \rho_t(\x)$ is called the ``score'' function. Diffusion models estimate the score function $\nabla_{\x} \log \rho_t(\x)$, using which one can sample from $\rho_1$ following \eqref{eq:sampling_with_v}. There also exists a closed-form expression of the score function as follows:
\begin{align}
    \nabla_{\x} \log \rho_t(\x) = \bbE_{\substack{\x_1 \sim \rho_1, \\ \epsilon \sim \cN(0, \bm I), \\ \widetilde{\x}_t = \x_1 + \sigma(t)\epsilon }} \left[ \frac{\widetilde{\x}_t - \x_1}{\sigma^2(t)} \Big| \widetilde{\x}_t = \x \right].
\end{align}

\noindent\textbf{Flow Matching.}
Flow matching generalizes diffusion models by allowing $\rho_0$ to be arbitrary density (in contrast to Gaussian in diffusion models), and following arbitrary path determined by a time-differentiable function $\J_t: \bbR^N \to \bbR^N$ (termed as stochastic interpolant) that satisfies:
\begin{align}
    \J_0(\x_0, \x_1) = \x_0, \quad \J_1(\x_0, \x_1) = \x_1.
\end{align}
If one samples $(\x_0, \x_1) \sim \rho(\x_0, \x_1)$, where $\rho(\x_0, \x_1)$ is a coupling distribution whose marginals are $\rho_0$ and $\rho_1$, then $\J_t(\x_0, \x_1)$ is a random variable whose density $\rho_t$ satisfies $\rho_{t=0} = \rho_0$ and $\rho_{t=1} = \rho_1$.
The vector field $\v_t$ corresponding to $\J_t$ and thus $\rho_t$ is given by 
\begin{align}\label{eq:fm_vf}
    \v_t(\x) = \bbE_{\substack{(\x_0, \x_1) \sim \rho}} \left[ \partial_t \J_t(\x_0, \x_1) \Big| \J_t(\x_0, \x_1) = \x \right]
\end{align}
where, $\partial_t$ denotes the partial derivative with respect to time i.e. $\frac{\partial}{\partial t}$.

Since the vector field $\v_t$ in \eqref{eq:fm_vf} is completely determined by $\rho$ and $\J_t$. Therefore, we use the notation $\v_t^{(\rho, \J)}$ to completely specify the vector field in consideration.

\subsection{Consistency Models}
Given any point $\x_0 \sim \rho_0$, consider the trajectory $\{ \x_t\}_{t \in [0,1]}$ given by 
\begin{align}
    \x_{t'} = \x_{t}+ \int_{t}^{t'} v_s(x_s) ds, \quad \forall t, t' \in [0,1]
\end{align}
Then $\x_1$ is the corresponding sample from $\rho_1$ obtained by following the learnt vector field $\v_t$. 
To address the slow sampling process of diffusion models, consistency models \cite{song2023consistency} was proposed to allow fast inference by estimating a function $\f_t: \bbR^N \to \bbR^N$ that can directly output the corresponding clean image $\x_1$ given any noisy image $\x_t$ along the trajectory. Given a score model $\s_t(\x)$ (estimate of $\nabla_{\x} \log \rho_t(\x)$), consistency model estimates $\f_t$ using the following optimization problem:
\begin{align}\label{eq:cm distillation}
    \min_{\f_t} ~ \cL_{\rm CM}(\f_t) := \bbE_{\substack{\x_1 \sim \rho_1, \bm \epsilon \sim \cN(0,1),\\ \widetilde{\x}_t= \sigma(t) \bm \epsilon + \x_1}} & \norm{\f_{t}\left(\widetilde{\x}_t\right)- \f_{t + \Delta t}\left(\widetilde{\x}_t +  \sigma(t) \s_t(\widetilde{\x}_t)\Delta t\right)}_2^2
\end{align}
For $\Delta t \to 0$, it was shown that the minimizer of $\cL_{\rm CM}$, $\f_t^\star$ is such that $\f^\star_t(\x_t) = \x_1$.

Problem \eqref{eq:cm distillation} requires a pre-trained score model $\s_t$. However, it was shown that one can use a sample estimate of the score given in \eqref{eq:diffusion_vf} in \eqref{eq:cm distillation}. A sample estimate of the score, $\widehat{\s_t}(\widetilde{\x}_t)$, follows the following sampling procedure:
\begin{align}
    \x_1 \sim \rho_1, ~ \bm \epsilon \sim \cN(0,\bm I), ~ \widetilde{\x}_t = \x_1 + \sigma(t) \bm \epsilon, ~~ \widehat{\s_t}(\widetilde{\x}_t) = \frac{\widetilde{\x}_t -\x_1}{\sigma^2(t)}.
\end{align}
Hence, the consistency loss becomes:
\begin{align}\label{eq:cm training}
    \min_{\f_t} ~ \widehat{\cL}_{\rm CM}(\f_t) := \bbE_{\substack{\x_1 \sim \rho_1, \bm \epsilon \sim \cN(0,1),\\ \widetilde{\x}_t= \sigma(t) \bm \epsilon + \x_1}} & \norm{\f_{t}\left(\widetilde{\x}_t\right)- \f_{t + \Delta t} \left(\widetilde{\x}_t +  \sigma(t) \widehat{\s}_t(\widetilde{\x}_t)\Delta t\right)}_2^2
\end{align}
It was shown in \cite{song2023consistency} that $\widehat{\cL}_{\rm CM}$ achieves the same minimizer as $\cL_{\rm CM}$ for $\Delta t \to 0$.

One can generalize $\cL_{\rm CM}$ for the case of arbitrary vector field $\v_t$ that connects any two densities $\rho_0$ and $\rho_1$ as follows \cite{dou2024unified}:
\begin{align}\label{eq:gcm}
    \min_{\f_t} ~ \cL_{\rm GCM}(\f_t, \v_t) := \bbE_{\substack{t \sim \text{U(0,1)},\\ (\x_0, \x_1) \sim \rho, \\ \widetilde{\x}_t = \J_t(\x_0, \x_1)}} & \norm{\f_{t}\left(\widetilde{\x}_t \right)- \f_{t+\Delta t} \left(\widetilde{\x}_t +  \v_t^{(\rho, \J)}(\widetilde{\x}_t)\Delta t \right)}_2^2.
\end{align}
As in diffusion models, one can obtain an empirical estimate of $\v_t$ as follows:
\begin{align}
    (\x_0, \x_1) \sim \rho, ~ \widetilde{\x}_t = \J_t(\x_0, \x_1), ~ \widehat{\v}_t^{(\rho, \J)}(\widetilde{\x}_t) = \partial_t \J_t(\x_0, \x_1). 
\end{align}
Note that for linear interpolant $\J_t = (1-t)\x_0 + t \x_1 $, $\widehat{\v}_t = \x_1 - \x_0$.
Hence the alternate objective is
\begin{align}
    \min_{\f_t} ~ \widehat{\cL}_{\rm GCM}(\f_t, \v^{(\rho, \J)}) := \bbE_{\substack{t \sim \text{U(0,1)},\\ (\x_0, \x_1) \sim \rho, \\ \widetilde{\x}_t = \J_t(\x_0, \x_1)}} & \norm{\f_{t}\left(\widetilde{\x}_t \right)- \f_{t+\Delta t} \left(\widetilde{\x}_t +  \widehat{\v}_t^{(\rho, \J)}(\widetilde{\x}_t)\Delta t \right)}_2^2
\end{align}

\begin{Fact}\label{fact:optimal_gcm}
    $\cL_{\rm GCM}(\f_t, \v^{(\rho,\J)})$ has a unique minimizer, represented by $\f_t^{(\rho, \J)}$, which is also the minimizer of $\widehat{\cL}_{\rm GCM}(\f_t, \v^{(\rho,\J)})$.
\end{Fact}
Fact \ref{fact:optimal_gcm} can be derived using the same proof technique as in \cite{song2023consistency}. $\widehat{\cL}_{\rm GCM}$ was explored for generative modeling in \cite{dou2024unified}.

\section{Flow based Distribution Matching}

\subsection{Problem Statement}
In many applications, e.g., domain translation and adaptation, it is often required to find a mapping between $\rho_0$ and $\rho_1$ that satisfies from problem-specific constraints. Consider the following optimization problem:
\begin{align}\label{eq:general_problem}
    \minimize_{\g} & ~ {\rm div}([\g]_{\# \rho_0} || \rho_1) \\
    \text{subject to} & ~ \g \in \cG, \nonumber
\end{align}
where ${\rm div}$ represents a distribution divergence operator such that ${\rm div}(p_1 || p_2)=0 \iff p_1 = p_2$ almost everywhere, $\cG$ is some function class specified by the problem, and $\g$ our optimization function that transports $\rho_0$ to $\rho_1$. { Note that although minimizer of $\cL_{\rm GCM}$, $\f_0^{(\rho, \J)}$ can minimize distribution divergence between given source and target distribution, $\f_0^{(\rho, \J)}$ may not be an element of $\cG$.} In the following, we present some examples of Problem \eqref{eq:general_problem}, along with reasons why $\cL_{\rm GCM}$ alone cannot solve \eqref{eq:general_problem}:

\noindent \textbf{Latent representations.} A simple example of \eqref{eq:general_problem} is when $\cG = \{\h: \bbR^D \to \bbR^N\}$ for some $D<N$. This example indicates that $\rho_0$ lives on a lower dimensional space compared to $\rho_1$. Such constraint is useful to find a lower dimensional semantic representation of data, e.g., StyleGAN \cite{karras2019style}. Here, the minimizer of $\cL_{\rm GCM}$, $\f_0^{(\rho, \J)}$ can only take input and output of the same dimension. Hence $\f_0^{(\rho, \J)}$ is not a solution to Problem \eqref{eq:general_problem}. 

\noindent \textbf{Linear Maps.} Another example is $\cG = \{\z \mapsto \A \z | \A \in \bbR^{N \times D}\}$, i.e., a class of linear maps. Such constraints are useful to embed prior information about the transport map. For example, linear map based distribution matching was found useful for unsupervised word translation between different languages \cite{conneau2017word}. However, $\f_0^{(\rho, \J)}$ is unlikely to be a linear map.

\noindent \textbf{Weakly Supervised domain translation.} Consider the setting, where a few paired samples $\{\z^{(i)}, \x_1^{(i)}\}_{i=1}^T$ are available along with a lot of unpaired samples (e.g., weakly supervised machine translation, image-to-image translation, etc). In that case, we would like $\cG = \{ \h : \h(\z^{(i)}) = \x^{(i)}, \forall i \in [T]\}$. Again, $\f_0^{(\rho, \J)}$ is unlikely to satisfy such constraints. 

In general, whenever $\f_0^{(\rho, \J)} \not \in \cG$, problem \eqref{eq:gcm} cannot be used to solve \eqref{eq:general_problem}. For high dimensional data such as images, Problem \eqref{eq:general_problem} is generally tackled using GAN based distribution matching. However, GAN poses a difficult bilevel min-max optimization objective that are known to be highly unstable. Moreover, GANs suffer from mode collapse issues. On the other hand, CNFs have a simple quadratic minimization objective that are stable and scalable. Hence, it is well-motivated to seek for CNF-like objectives to solve Problem \eqref{eq:general_problem}.

\subsection{Proposed Solution}
We propose the following objective in order to solve Problem \eqref{eq:general_problem}:
\begin{align} \label{eq:gcm_gen}
    \min_{\g, \h, \f_t} & ~ \cL_{\rm GCM} (\f_t, {\v}^{(\rho, \J)})  \\
    \text{s.t.} &~ \f_0 \circ \h \in \cG, \nonumber \\
                        &~ \g = \f_0 \circ \h \nonumber \\
                        &~ \rho \in \Pi([\h]_{\# \rho_0}, \rho_1),\nonumber 
\end{align}
where, $\rho$ is the coupling distribution, $\Pi([\h]_{\# \rho_0}, \rho_1)$ denotes the set of joint distribution whose marginals are $[\h]_{\# \rho_0}$ and $\rho_1$, and  $\f_0$ denotes $\f_{t=0}$. Essentially, the vector field connects the distributions $[\h]_{\# \rho_0}$ with $\rho_1$, instead of $\rho_0$ with $\rho_1$. With Problem \eqref{eq:gcm_gen}, one can show the following:

\begin{proposition}
    Suppose that there exists a $\g^\star$ such that $\g^\star \in \cG$ and ${\rm div}([\g^{\star}]_{\# \rho_0} || \rho_1) = 0$. 
    Let $\widehat{\g}, \widehat{\h}, \widehat{\f}_t$ be a solution to \eqref{eq:gcm_gen}. Then $\widehat{\g}$ is also a solution to Problem \eqref{eq:general_problem}. 
\end{proposition}
\begin{proof}
    From Fact \ref{fact:optimal_gcm}, $\f_t^{(\rho, \J)}$ achieves $\cL_{\rm GCM}(\f_t^{(\rho, \J)}, {\v}^{(\rho, \J)}) = 0$. Next, set $\h^\star = (\f_0^{(\rho, \J)})^{-1} \circ \g^\star$. Then the triplet $\g^\star, \h^\star, \f_t^{(\rho, \J)}$ is a solution to Problem \eqref{eq:gcm_gen} that achieves zero loss. 
    This implies that $\widehat{\f}_t$ also achieves zero $\cL_{\rm GCM}$, i.e., 
    \begin{align*}
    {\rm div}([\widehat{\f}_0 \circ \widehat{\h}]_{\# \rho_0} || \rho_1) = 0.
    \end{align*}
    Hence, $\widehat{\g} = \widehat{\f}_0 \circ \widehat{\h}$ is a solution to Problem \eqref{eq:general_problem}.
\end{proof}

\noindent \textbf{Reformulation of \eqref{eq:gcm_gen}.} The objective in \eqref{eq:general_problem} is a bit hard to deal with since there are three neural networks $\g, \h, \f_t$ involved. Hence, we consider the following reformulation
\begin{align}\label{eq:gcm_gen_reform}
    \min_{\f_t, \g} &~ \mathcal{L}_{\text{FDM}}(\f_t, \g) =& \bbE_{\substack{t \sim \text{U(0,1)},\\ \z, \x_1 \sim \rho( \g(\z), \x_1),\\ \widetilde{\x}_t=\J_t(\g(\z), \x_1)}} &\underbrace{\norm{\f_{t}\left(\widetilde{\x}_t, t \right) - \f_{t + \Delta t}\left(\widetilde{\x}_t + \partial_t \J_t(\g(\z), \x_1) \Delta t \right)}_2^2}_{\text{Consistency Loss}} + \nonumber\\
    & & \underbrace{\norm{ \g(\mathbf{z}) - \f_{0}(\g(\mathbf{z}))}_2^2}_{\text{Generator loss}}, \\
    \text{s.t.} &~ \g \in \cG. \nonumber
\end{align}

\begin{proposition}\label{prop:main}
    Suppose that there exists a $\g^\star$ such that $\g^\star \in \cG$ and ${\rm div}([\g^\star]_{\# \rho_0} || \rho_1) = 0$. Let $\rho, \J_t$ be such that $\f_0^{\rho, \J} = {\rm Id}$ when $[\g]_{\# \rho_0} = \rho_1$, where {\rm Id} denotes the identity function. Let $\widehat{\g},\widehat{\f}_t$ be a solution to \eqref{eq:gcm_gen_reform}. Then $\widehat{\g}$ is also a solution to Problem \eqref{eq:general_problem}. 
\end{proposition}

\begin{proof}
    Note that the pair $\g^\star, \f_t^{(\rho, \J)}$ achieves $\cL_{\rm FDM} = 0$ and $\g^\star \in \cG$. This is because $[\g^\star]_{\# \rho_0} = \rho_1$ which implies that $\f_t^{(\rho, \J)} = {\rm Id}$. Therefore, both loss terms in $\cL_{\rm FDM}$ is zero. 
    
    Hence the optimal solution $\widehat{\g}$ and $\widehat{\f}_t$ should be such that 
    \begin{align*}
        \widehat{\g}(\z) &= \widehat{\f}_0(\widehat{\g}(\z)), \forall \z \\
        \implies \widehat{\f}_0 &= {\rm Id}.
    \end{align*}
    Finally, $\widehat{\f}_0 = {\rm Id}$ combined with the fact that consistency loss term is zero implies that $[\widehat{\g}]_{\# \rho_0} = \rho_1$. Hence $\widehat{\g}$ is a solution to Problem \eqref{eq:general_problem}. 
\end{proof}
Above proposition says that for certain choice of coupling $\rho$ and interpolant $\J_t$, Problem \eqref{eq:gcm_gen_reform} is equivalent to Problem \ref{eq:general_problem}. An example of a coupling $\rho$ that satisfies the condition in \ref{prop:main} is the optimal transport coupling. In practice, optimal transport coupling is computationally expensive to compute, hence minibatch optimal transport is often utilized as a proxy \cite{tong2023improving, pooladian2023multisample}.

\noindent \textbf{Algorithm and Practical Implementation.}
To stabilize the training, we first reformulate \eqref{eq:gcm_gen_reform} as follows:
\begin{align}
    \min_{\f_t, \g} ~ \mathcal{L}_{\text{FDM}}(\f_t, \g) = \bbE_{\substack{t \sim \text{U(0,1)},\\ \z, \x_1 \sim \rho( \g(\z), \x_1),\\ \widetilde{\x}_t=\J_t(\g(\z), \x_1)}} &\underbrace{\norm{\f_{t}\left(\widetilde{\x}_t, t \right) - \f_{t + \Delta t}^-\left(\widetilde{\x}_t + \partial_t \J_t(\g^-(\z), \x_1) \Delta t \right)}_2^2}_{\text{Consistency Loss}} + \nonumber\\
    &\underbrace{\norm{ \g(\mathbf{z}) - \f_{0}^-(\g^-(\mathbf{z}))}_2^2}_{\text{Generator loss}},
\end{align}

where $\g^-$ and $\f_t^-$ represent the copies of $\g$ and $\f_t$, respectively, and are treated as constant during gradient-based optimization. Note that we still have moving targets for both $\f_t$ and $\g$ in the consistency loss and generator loss terms, respectively. Hence we alternately optimize $\f_t$ and $\g$ with respect to consistency and generator losses, respectively. Mainly, each iteration of our algorithm consists of the following two sub-optimization steps.
\begin{enumerate}
    \item Optimize $\f_t$ with gradient-based optimizer for $T_f$ iterations
    \item Optimize $\g$ with gradient-based optimizer for $T_g$ iterations.
\end{enumerate}

\section{Simulations}

\subsection{Synthetic Data Experiments}
In order to verify the soundness of our proposed method, we first ran our method in some simulated 2D-data distributions. The purpose is to transform the 2D source distribution using a function $\g$ so that the resulting 2D distribution matches the target distribution. This corresponds to the unconstrained case of Problem \eqref{eq:general_problem}. Note that the consistency model $\f_t$ is used as an auxiliary tool for training $\g$. The results in figure \ref{fig: distribution matching result in simulated data} demonstrates that our proposed method is able to approximately match the target distribution.

\begin{figure}[h!]
    \centering
    \begin{subfigure}{0.49\textwidth}
        \includegraphics[width=\textwidth]{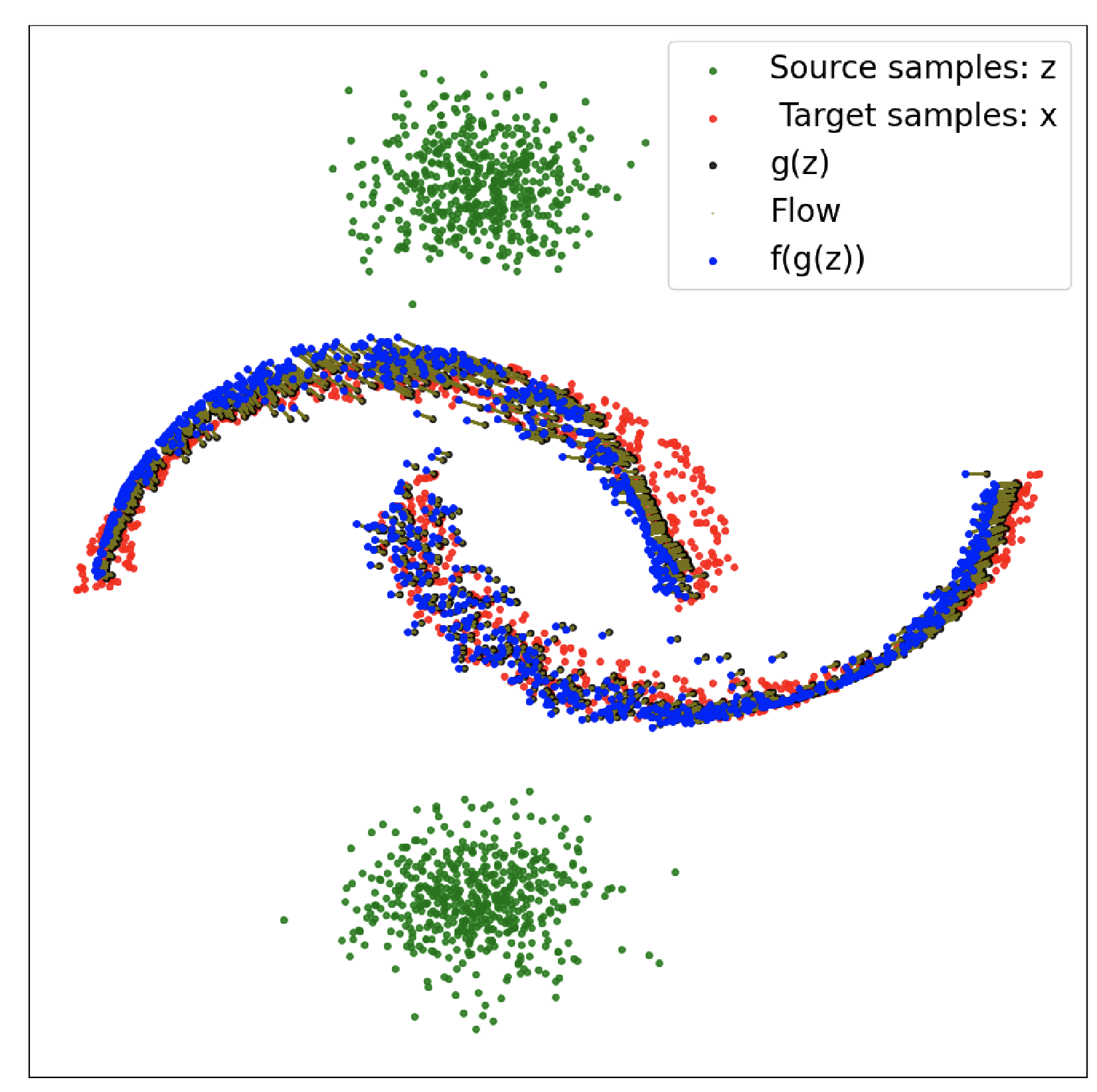}
        \caption{$\rho_0$: 2 Gaussians,\\ $\rho_1$: 2 Moons}
    \end{subfigure}
     \begin{subfigure}{0.49\textwidth}
        \includegraphics[width=\textwidth]{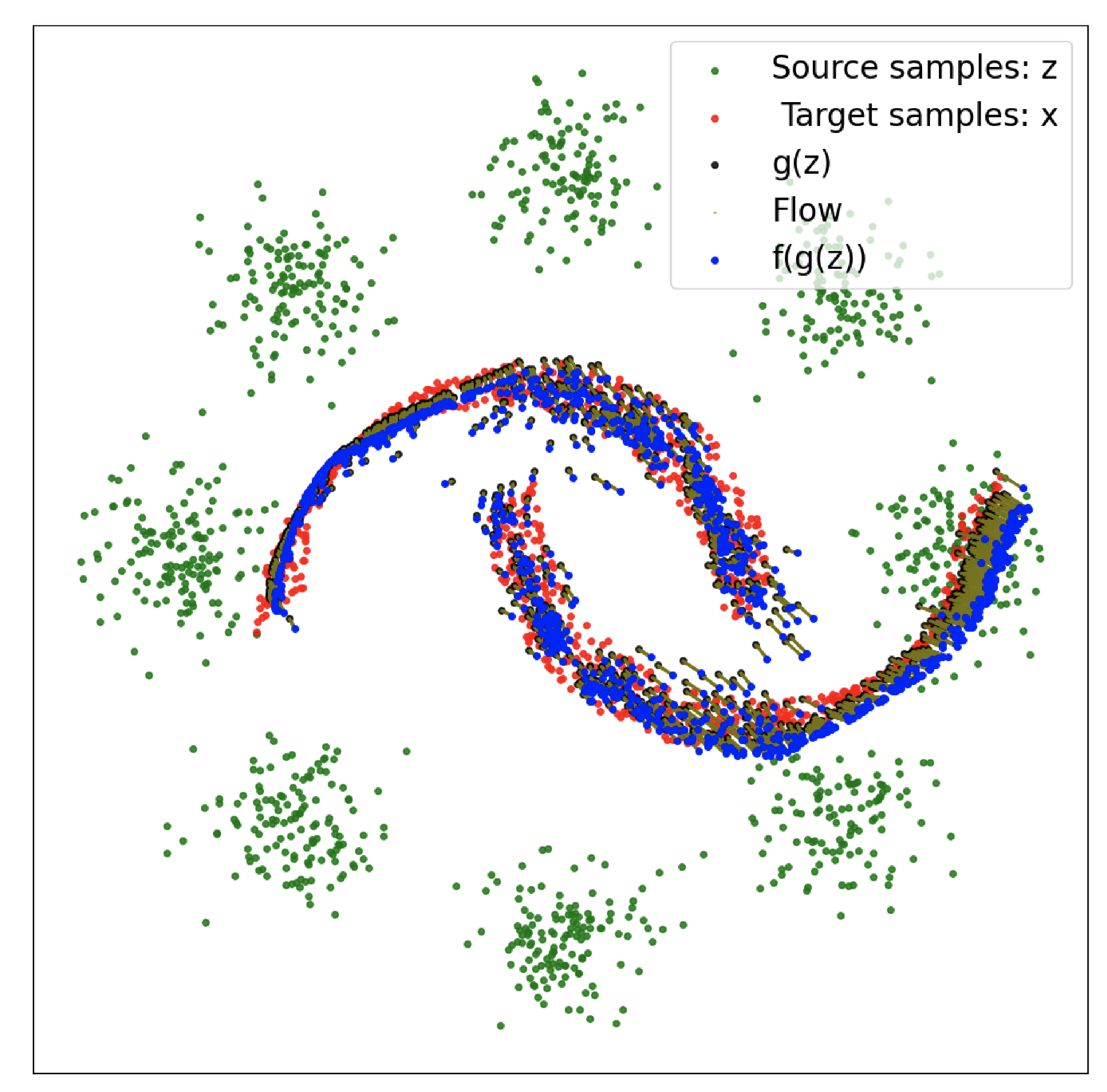}
        \caption{$\rho_0$: 8 Gaussians,\\ $\rho_1$: 2 Moons}
    \end{subfigure}
    \begin{subfigure}{0.49\textwidth}
        \includegraphics[width=\textwidth]{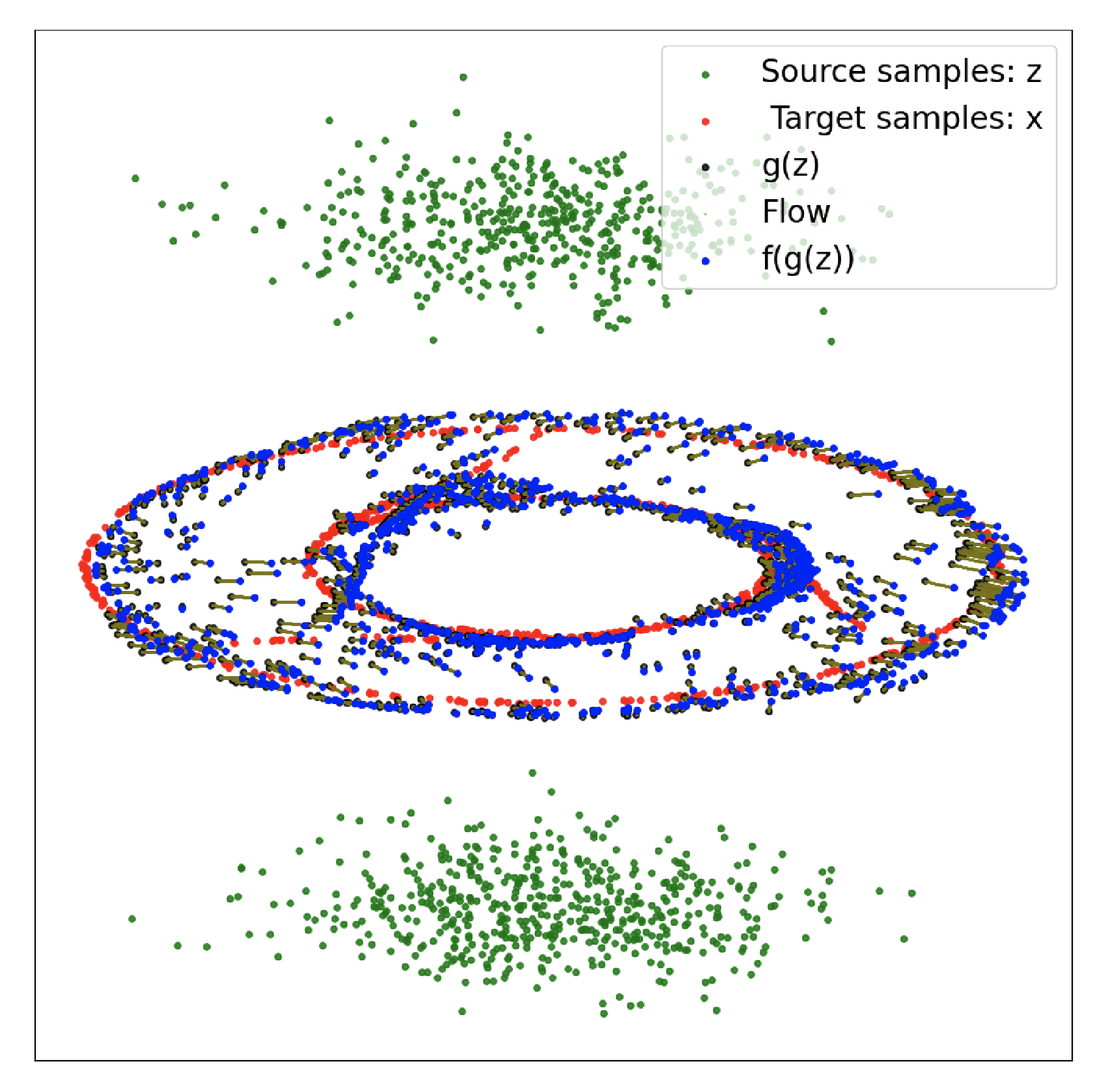}
        \caption{$\rho_0$: 2 Gaussians,\\ $\rho_1$: DiffEqML Logo}
    \end{subfigure}
    \caption{Distribution Matching by our proposed method in 2D simulated data distribution\\ Green: Source Distribution$(\z \sim \rho_0)$, Red: Target Distribution ${\x \sim \rho_1}$, Black: $\g(\z)$, Blue: $\f_0(\g(\z))$ }
    \label{fig: distribution matching result in simulated data}
\end{figure}

\subsection{Real Data Experiments}

\begin{figure}[!h]
    \centering
    \includegraphics[width=0.9\linewidth]{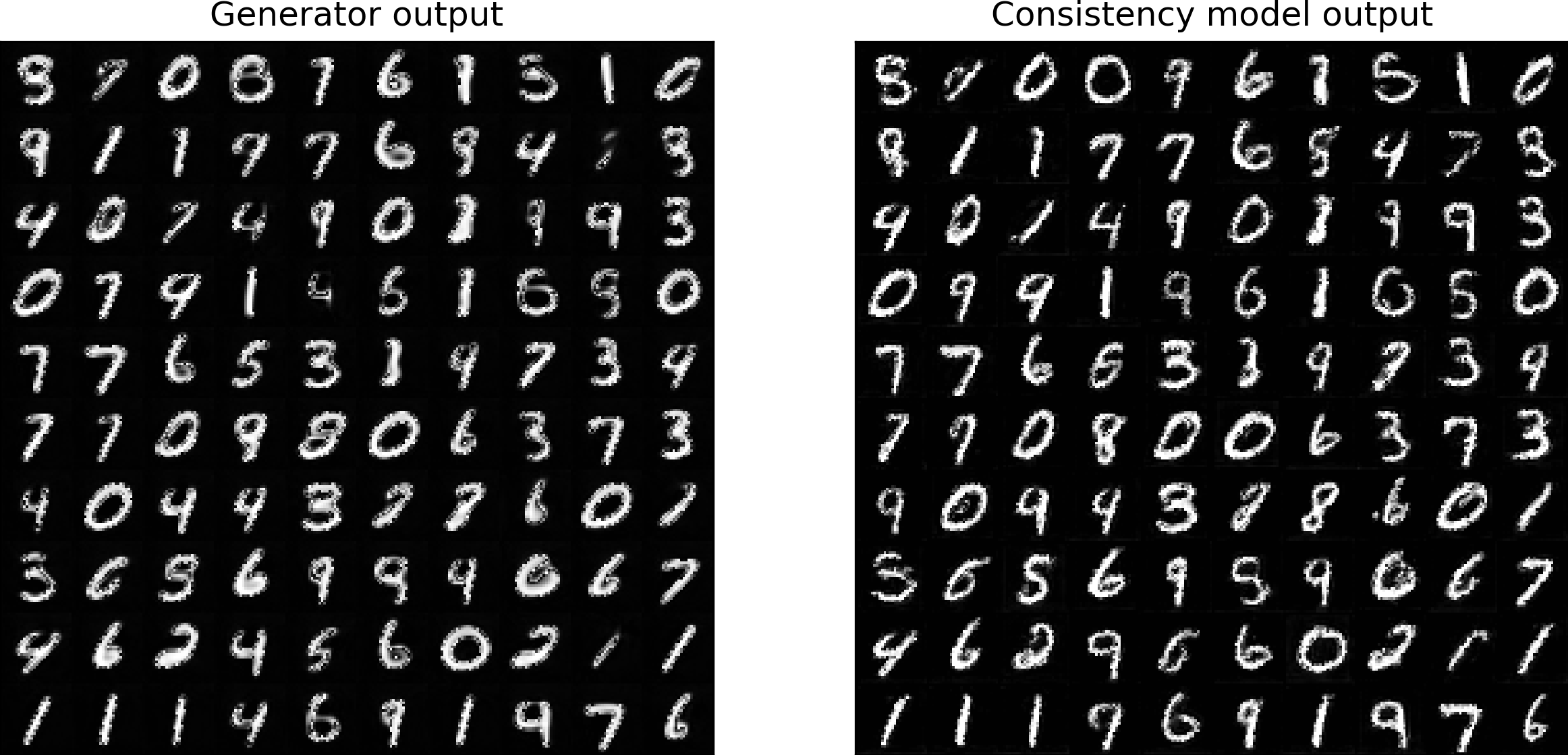}
    \caption{[Left] Random samples from $\g(\z), \z \in \bbR^{256}$. [Right] Random samples from $\f_0(\g(\z))$.}
    \label{fig:mnist}
\end{figure}

We verify the proposed method on a simple image dataset: MNIST digits. We use a convolutional 
generator with input dimension of 256 and output dimension of $28 \times 28$ to represent $\g$. This represents the case where $\cG = \{\bbR^{256} \to \bbR^{28 \times 28}\}$ in Problem \eqref{eq:general_problem}. We use a UNet from \cite{tong2023improving} to represent $\f_t$. We set $T_f = 50$ and $T_g = 20$ and train the neural networks for 100 iterations. Figure \ref{fig:mnist} shows the qualitative result of our method. One can see that the distribution of $\g(\z)$ approximate the MNIST digit distribution.

\section{Discussions and Conclusion}
In this project, we proposed a formulation to perform distribution matching using flow matching. While possessing flexibility in latent manipulation as GAN does, it avoid well-known difficulties in training GAN by relying on a different principle from flow matching and flow consistency property.

Our method showed a promising result on 2D simulated dataset and MNIST dataset. We also demonstrated that our method can be used to train a latent generative model ($\g$), with the dimension of source distribution being much less compared to the target distribution. Due to time constraints, high resolution image datasets could not be explored during the project timeframe.
We believe that with proper tuning and sufficient training, our method should work well in such high-resolution datasets as well. We intend to extend these experiments and further improve our method in the future.

\clearpage

\pagebreak
\bibliographystyle{plainnat}
\bibliography{main}

\end{document}